\documentclass{article}

\usepackage{microtype}
\usepackage{graphicx}
\usepackage{subfigure}
\usepackage{booktabs} % for professional tables

\usepackage{hyperref}

\usepackage[accepted]{icml2020}

\usepackage{times}
\usepackage{helvet}
\usepackage{courier}
\usepackage{amssymb}
\usepackage{amsmath}
\usepackage{amsthm}
\usepackage{latexsym}
\usepackage{graphicx}
\usepackage{wrapfig}
\usepackage{color,colortbl}
\usepackage{algorithm,algorithmic}
\usepackage{url}
\newtheorem{thm}{Theorem}%[section]

\newtheorem{prop}[thm]{Proposition}

\newtheorem{defn}[thm]{Definition}

\newtheorem*{thm*}{Proposition}
\newtheorem*{prop*}{Proposition}
\newtheorem*{lem*}{Lemma}

\newcommand{\Fig}[1]{Figure~\ref{#1}}
\newcommand{\Sec}[1]{Section~\ref{#1}}
\newcommand{\Tab}[1]{Table~\ref{#1}}

\newcommand{\Eqn}[1]{Eq.~(\ref{#1})}
\newcommand{\Eqs}[2]{Eqs.~(\ref{#1}-\ref{#2})}

\newcommand{\Thm}[1]{Theorem~\ref{#1}}
\newcommand{\Prop}[1]{Proposition~\ref{#1}}

\newcommand{\Algo}[1]{Algorithm~\ref{#1}}

\definecolor{Gray}{gray}{0.8}
\newcommand{\red}[1]{\textcolor{red}{#1}}
\renewcommand{\hat}{\widehat}
\renewcommand{\tilde}{\widetilde}
\renewcommand{\>}{{\rightarrow}}

\DeclareMathOperator*{\argmin}{argmin}

\newcommand{\rank}{\operatorname{rank}}

\newcommand{\R}{{\mathbb R}}
\newcommand{\Z}{{\mathbb Z}}

\renewcommand{\P}{{\mathbf P}}
\newcommand{\E}{{\mathbf E}}

\newcommand{\1}{{\mathbf 1}}
\newcommand{\0}{{\mathbf 0}}
\newcommand{\ip}{{\xrightarrow{P}}}

\newcommand{\F}{{\mathcal F}}

\renewcommand{\L}{{\mathbf L}}
\newcommand{\cL}{{\mathcal L}}

\newcommand{\Q}{{\mathbf Q}}

\newcommand{\T}{{\mathbf T}}

\newcommand{\V}{{\mathbf V}}
\newcommand{\W}{{\mathbf W}}
\newcommand{\X}{{\mathcal X}}
\newcommand{\Y}{{\mathcal Y}}

\renewcommand{\a}{{\mathbf a}}
\renewcommand{\b}{{\mathbf b}}

\newcommand{\f}{{\mathbf f}}
\newcommand{\h}{{\mathbf h}}%\mathbf 

\newcommand{\p}{{\mathbf p}}
\newcommand{\q}{{\mathbf q}}

\renewcommand{\u}{{\mathbf u}}

\newcommand{\x}{{\mathbf x}}
\newcommand{\y}{{\mathbf y}}
\newcommand{\zo}{\textup{\textrm{0-1}}}
\newcommand{\acc}{\textup{\textrm{acc}}}
\newcommand{\er}{\textup{\textrm{er}}}
\newcommand{\regret}{\textup{\textrm{regret}}}
\newcommand{\decode}{\textup{\textrm{decode}}}

\newcommand{\balpha}{{\boldsymbol \alpha}}

\icmltitlerunning{Convex Calibrated Surrogates for the Multi-Label F-Measure}

\begin{document}

%---------- Title -------------

\twocolumn[
\icmltitle{Convex Calibrated Surrogates for the Multi-Label F-Measure}



\icmlsetsymbol{equal}{*}

\begin{icmlauthorlist}
\icmlauthor{Mingyuan Zhang}{upenn}
\icmlauthor{Harish G. Ramaswamy}{iitm}
\icmlauthor{Shivani Agarwal}{upenn}
\end{icmlauthorlist}

%\icmlaffiliation{upenn}{University of Pennsylvania, Philadelphia, PA, USA}
%\icmlaffiliation{upenn}{Dept.\ of Computer and Information Science, University of Pennsylvania, Philadelphia, USA}
\icmlaffiliation{upenn}{Department of Computer and Information Science, University of Pennsylvania, Philadelphia, PA, USA}
%\icmlaffiliation{iitm}{Indian Institute of Technology Madras, Chennai, India}
%\icmlaffiliation{iitm}{Dept.\ of Computer Science and Engineering, Indian Institute of Technology Madras, Chennai, India}
\icmlaffiliation{iitm}{Department of Computer Science and Engineering, Indian Institute of Technology Madras, Chennai, India}

\icmlcorrespondingauthor{Shivani Agarwal}{ashivani@seas.upenn.edu}

%\icmlkeywords{Machine Learning, ICML}
\icmlkeywords{Multi-label Learning, Multi-Label Classification, F-measure, Surrogate Losses, Output Coding}

\vskip 0.3in
]



\printAffiliationsAndNotice{}  % leave blank if no need to mention equal contribution

%----------- Abstract -----------

\begin{abstract}
%\vspace{-8pt}
The $F$-measure is a widely used performance measure for multi-label classification,
% (MLC), 
where multiple labels can be active in an instance simultaneously
%(an example of an MLC problem is image tagging, where multiple tags can be active in an image). 
(e.g.\ in image tagging, multiple tags can be active in any image). 
In particular, the $F$-measure explicitly balances recall
(fraction of active labels predicted to be active) 
and precision 
(fraction of labels predicted to be active that are actually so),
%, both desirable properties of an MLC classifier. 
%of an MLC classifier, 
both of which are important in evaluating the overall performance of 
%an MLC 
a multi-label 
classifier. 
As with most discrete prediction problems, however, directly optimizing the $F$-measure is computationally hard. 
In this paper, we explore the question of designing convex surrogate losses that are \emph{calibrated} for the $F$-measure -- specifically, that have the property that minimizing the surrogate loss yields (in the limit of sufficient data) a Bayes optimal multi-label classifier for the $F$-measure.
We show that the $F$-measure for an $s$-label problem, when viewed as a $2^s\times 2^s$ loss matrix, has rank at most $s^2+1$, and apply a result of \citet{Ramaswamy+14} to design a family of convex calibrated surrogates for the $F$-measure. 
The resulting surrogate risk minimization algorithms can be viewed as decomposing the multi-label $F$-measure learning problem into $s^2+1$ binary class probability estimation problems. We also provide a quantitative regret transfer bound for our surrogates, which allows any regret guarantees for the binary problems to be transferred to regret guarantees for the overall $F$-measure problem, and discuss a connection with the algorithm of \citet{Dembczynski2013}. Our experiments confirm our theoretical findings.  
\end{abstract}

\vspace{-16pt}
%========== SECTION 1 ===========
\section{Introduction}
\label{sec:intro}
\vspace{-4pt}

The $F_\beta$-measure is a widely used performance measure for multi-label classification (MLC) problems. 
In particular, in an MLC problem, multiple labels can be active in an instance simultaneously; a good example is that of image tagging, where several tags (such as \texttt{sky}, \texttt{sand}, \texttt{water}) can be active in the same image. 
In such problems, when evaluating the performance of a classifier on a particular instance, it is important to balance the \emph{recall} of the classifier on the given instance, i.e.\ the fraction of active labels for that instance that are correctly predicted as such, and the \emph{precision} of the classifier on the instance, i.e.\ the fraction of labels predicted to be active for that instance that are actually so. The $F_\beta$-measure accomplishes this by taking the (possibly weighted) harmonic mean of these two quantities. 

Unfortunately, as with most discrete prediction problems, optimizing the $F_\beta$-measure directly during training is computationally hard. 
Consequently, one generally settles for some form of approximation. 
One approach is to simply treat the labels as independent, and train a separate binary classifier for each label; this is sometimes referred to as the \emph{binary relevance} (BR) approach. Of course, this ignores the fact that labels can have correlations among them (e.g.\ \texttt{sky} and \texttt{cloud} may be more likely to co-occur than \texttt{sky} and \texttt{computer}). Several other approaches have been proposed in recent years \cite{Dembczynski2013,Koyejo2015,Wu2017,Pillai2017}.

In this paper, we turn to the theory of convex calibrated surrogate losses -- which has yielded convex risk minimization algorithms for several other discrete prediction problems in recent years \cite{Bartlett+06,Zhang04b,TewariBa07,Steinwart07,Duchi+10,Gao2013,Ramaswamy+14,Ramaswamy+15} -- to design principled surrogate risk minimization algorithms for the multi-label $F_\beta$-measure.
In particular, for an MLC problem with $s$ tags, the total number of possible labelings of an instance is $2^s$ 
%(each of the $s$ tags can be active or inactive). 
(each tag can be active or inactive). 
Viewing the $F_\beta$-measure as (one minus) a $2^s \times 2^s$ loss matrix, we show that this matrix has rank at most $s^2+1$, and apply the results of %Ramaswamy and Agarwal 
\citet{Ramaswamy+14} to design an output coding scheme that reduces the $F_\beta$ learning problem to a set of $s^2+1$ binary class probability estimation (CPE) problems. By using a suitable binary surrogate risk minimization algorithm (such as binary logistic regression) 
%for each of these binary problems, 
%for each binary problem, 
for these binary problems, 
we effectively construct a $(s^2+1)$-dimensional convex calibrated surrogate loss for the $F_\beta$-measure. 
We also 
give a quantitative regret transfer bound for the constructed surrogate, which allows us to transfer any regret guarantees for the binary subproblems to guarantees on $F_\beta$-regret for the overall MLC problem. In particular, this means that 
%if we use a consistent learner for the binary problems, then we obtain a consistent learner for the MLC problem (whose $F_\beta$-regret goes to zero as the training sample size increases). 
using a consistent learner for the binary problems yields a consistent learner for the MLC problem (whose $F_\beta$-regret goes to zero as the training sample size increases). 

Our algorithm is related to the plug-in algorithm of \citet{Dembczynski2013}, which also estimates $s^2+1$ statistics of the underlying distribution. \citet{Dembczynski2013} estimate these statistics by reducing the $F_\beta$ maximization problem to $s$ multiclass CPE problems, each with at most $s+1$ classes (plus one binary CPE problem); we do so by reducing the problem to $s^2+1$ binary CPE problems.
As we show, both algorithms effectively estimate the same $s^2+1$ statistics, and indeed, both perform similarly in experiments. 
Interestingly, 
the algorithm of \citet{Dembczynski2013}, while motivated primarily by the plug-in approach, can also be viewed as minimizing a certain convex calibrated surrogate loss (different from ours); conversely, our algorithm, while motivated primarily by the convex calibrated surrogates approach, can also be viewed as a plug-in algorithm.
Our study brings out interesting connections between the two approaches; in addition, to the best our knowledge, our analysis is the first to provide a quantitative regret transfer bound for calibrated surrogates for the $F_\beta$-measure.

\textbf{Organization.}
\Sec{sec:related-work} discusses related work.
\Sec{sec:prelim} gives preliminaries and background.
\Sec{sec:surrogates} gives our convex calibrated surrogates for the $F_\beta$-measure; \Sec{sec:regret-bound} provides a regret transfer bound for them.
\Sec{sec:relationship} discusses the relationship with the plug-in algorithm of \citet{Dembczynski2013}.
%\Sec{sec:expts} provides experimental evaluations of our algorithm.
\Sec{sec:expts} summarizes our experiments.
%\Sec{sec:concl} concludes with a brief discussion.

\vspace{-8pt}
%========== SECTION 2 ===========
\section{Related Work}
\label{sec:related-work}
\vspace{-4pt}

There has been much work on multi-label learning, learning with the $F_\beta$-measure, and convex calibrated surrogates. Below we briefly discuss work that is most related to our study. For detailed surveys on multi-label learning, we refer the reader to \citet{Zhang2014} and \citet{Pillai2017}.

%--------------------------
%\subsection{Multi-Label Learning}

\textbf{Bayes optimal multi-label classifiers.}
In an elegant study, \citet{Dembczynski2011} studied in detail the form of a Bayes optimal multi-label classifier for the $F_1$-measure. In particular, they showed that, for an $s$-label MLC problem, given a certain set of $s^2+1$ statistics of the true conditional label distribution (distribution over $2^s$ labelings), one can compute a Bayes optimal classifier for the $F_1$-measure in $O(s^3)$ time. 
Their result extends to general $F_\beta$-measures.
%Bayes optimal classifiers have also been studied for other MLC performance measures; e.g.\ \citet{Dembczynski2010} studied Bayes optimal classifiers for three MLC performance measures: Hamming loss, rank loss, and subset 0-1 loss. 
%They showed that for the first two, knowing the conditional distributions of individual labels suffice, while  
Bayes optimal classifiers have also been studied for other MLC performance measures, such as Hamming loss and subset 0-1 loss \cite{Dembczynski2010}.

\textbf{Consistent algorithms for multi-label learning.}
\citet{Dembczynski2013} extended and operationalized the results of \citet{Dembczynski2011} by providing a consistent plug-in MLC algorithm for the $F_\beta$-measure.
Specifically, they showed that the $s^2+1$ statistics of the conditional label distribution needed to compute a Bayes optimal classifier can be estimated via $s$ multiclass CPE problems, each with at most $s+1$ classes, plus one binary CPE problem; the statistics estimated by solving these CPE problems can then be plugged into the $O(s^3)$-time procedure of \citet{Dembczynski2011} to produce a consistent plug-in algorithm termed the \emph{exact F-measure plug-in} (EFP) algorithm.
Consistent learning algorithms have also been studied for other multi-label performance measures \cite{Gao2013,Koyejo2015}.\footnote{Note that while the study of \citet{Koyejo2015} also includes the $F_\beta$-measure (among other performance measures), their study is in the context of what has been referred to as the `expected utility maximization' (EUM) framework; in contrast, our study is in the context of what has been referred to as the `decision-theoretic analysis' (DTA) framework. Their results are generally incomparable to ours. (In particular, under the EUM framework, \citet{Koyejo2015} showed that a thresholding approach leads to Bayes optimal performance; on the contrary, under the DTA framework, it was shown by \citet{Dembczynski2011} that a thresholding approach cannot be optimal for general distributions.)}
%for the Hamming loss, the simple approach of learning an independent binary classifier for each of the $s$ labels, known as \emph{binary relevance} (BR), is known to 
The simple approach of learning an independent binary classifier for each of the $s$ labels, known as \emph{binary relevance} (BR), is known to yield a consistent algorithm for the Hamming loss; it also yields a consistent algorithm for the $F_\beta$-measure under the assumption of conditionally independent labels, but can be arbitrarily bad otherwise \cite{Dembczynski2011}.

\textbf{Large-margin algorithms for multi-label learning.}
Several studies have considered large-margin algorithms for multi-label learning with the $F_\beta$-measure. These include the \emph{reverse multi-label} (RML) and \emph{sub-modular multi-label} (SML) algorithms of \citet{PettersonCa10,PettersonCa11}, which make use of the StructSVM framework \cite{Tsochantiridis+05}, and more recently, the \emph{label-wise and instance-wise margin optimization} (LIMO) algorithm due to \citet{Wu2017}, which aims to simultaneously optimize several different multi-label performance measures. The RML and SML algorithms were proven to be inconsistent for the $F_\beta$-measure and shown to be outperformed by the EFP algorithm by \citet{Dembczynski2013}. We include a comparison with LIMO in our experiments.

\textbf{Multivariate $F_\beta$-measure for binary classification.}
The $F_\beta$-measure is also used as a multivariate performance measure in binary classification tasks with significant class imbalance. This use of the $F_\beta$-measure is related to, but distinct from, the use of the $F_\beta$-measure in MLC problems. Several approaches have been proposed that aim to optimize the multivariate $F_\beta$-measure in binary classification \cite{Joachims05,Ye2012,Parambath2014}.

%--------------------------
%\subsection{Convex Calibrated Surrogates}

\textbf{Convex calibrated surrogates.}
% for low-rank loss matrices.}
Convex surrogate losses are frequently used in machine learning to design computationally efficient learning algorithms.
% is now commonplace in machine learning.
The notion of calibrated surrogate losses, which ensures that minimizing the surrogate loss can (in the limit of sufficient data) recover a Bayes optimal model for the target discrete loss, was initially studied in the context of binary classification \cite{Bartlett+06,Zhang04a} and multiclass 0-1 classification \cite{Zhang04b,TewariBa07}. In recent years, calibrated surrogates have been designed for several more complex learning problems, including general multiclass problems and certain types of subset ranking and multi-label problems \cite{Steinwart07,Duchi+10,Gao2013,Ramaswamy+13,Ramaswamy+14,Ramaswamy+15}. 
In our work, we will make use of a result of \citet{Ramaswamy+14}, who designed convex calibrated surrogates based on output coding for multiclass problems with low-rank loss matrices.

\vspace{-4pt}
%========== SECTION 3 ===========
\section{Preliminaries and Background}
\label{sec:prelim}
%\vspace{-2pt}

%-----------------
\subsection{Problem Setup}

\textbf{Multi-label classification (MLC).}
In 
%a multi-label classification (MLC)
an MLC problem, there is an instance space $\X$, and a set of $s$ labels or `tags' $\cL=[s]:=\{1,\ldots,s\}$ that can be associated with each instance in $\X$. For example, in image tagging, $\X$ is the set of possible images, and $\cL$ is a set of $s$ pre-defined tags (such as \texttt{sky}, \texttt{cloud}, \texttt{water} etc) that can be associated with each image. 
The learner is given a training sample $S=\{(x_1,\y_1),\ldots,(x_m,\y_m)\} \in (\X\times\{0,1\}^s)^m$, 
%where each $\y_i\in\{0,1\}^s$ is a binary vector indicating 
where the labeling $\y_i\in\{0,1\}^s$ indicates 
which of the $s$ tags are active in instance $x_i$ (specifically, $y_{ij}=1$ denotes that tag $j$ is active in instance $x_i$, and $y_{ij}=0$ denotes it is inactive).
The goal is to learn from these examples a multi-label classifier $\h:\X\>\{0,1\}^s$ which, given a new instance $x\in\X$, predicts which tags are active or inactive via $\h(x)\in\{0,1\}^s$.

\textbf{$F_\beta$-measure.}
For any $\beta>0$, the $F_\beta$-measure evaluates the quality of an MLC prediction as follows. Given a true labeling $\y\in\{0,1\}^s$ and a predicted labeling $\hat{\y}\in\{0,1\}^s$, the recall and precision are given by
\vspace{-4pt}
\begin{eqnarray*}
\text{rec}(\y,\hat{\y}) = \frac{\sum_{j=1}^s y_j \hat{y}_j}{\|\y\|_1}
\,; 
& \hspace{-8pt} &
\text{prec}(\y,\hat{\y}) = \frac{\sum_{j=1}^s y_j \hat{y}_j}{\|\hat{\y}\|_1}
\,.
%\\[-14pt]
\end{eqnarray*}
In words, the recall measures the fraction of active tags that are predicted correctly, and the precision measures the fraction of tags predicted as active that are actually so. The $F_\beta$-measure balances these two quantities by taking their (weighted) harmonic mean:
\vspace{-4pt}
\begin{eqnarray}
F_\beta(\y,\hat{\y}) 
	& \hspace{-8pt} = & \hspace{-8pt} 
%	\bigg( 
%		\Big( {\textstyle{\frac{\beta^2}{1 + \beta^2}}} \Big) \frac{1}{\text{Recall}(\y,\hat{\y})} \, +
%\nonumber
%\\
%	& & \hspace{1.5cm}
%		\Big( {\textstyle{\frac{1}{1+\beta^2}}} \Big) \frac{1}{\text{Precision}(\y,\hat{\y})} 
%	\bigg)^{-1}
%\nonumber
%\\
	\bigg( 
		\Big( {\textstyle{\frac{\beta^2}{1 + \beta^2}}} \Big) \frac{1}{\text{rec}(\y,\hat{\y})} \, +
		\Big( {\textstyle{\frac{1}{1+\beta^2}}} \Big) \frac{1}{\text{prec}(\y,\hat{\y})} 
	\bigg)^{-1}
\nonumber
\\
	& \hspace{-8pt} = &  \hspace{-8pt}
	\frac{(1+\beta^2) \sum_{j=1}^s y_j \hat{y}_j}{\beta^2 \|\y\|_1 + \|\hat{\y}\|_1}
	\,.
\label{eqn:F-beta}
\end{eqnarray}
Clearly, $0 \leq F_\beta(\y,\hat{\y}) \leq 1$.
Higher values of the $F_\beta$-measure correspond to better quality predictions. We will take $\frac{0}{0}=1$, so that when $\y=\hat{\y}=\0$, 
%the $F_\beta$-measure takes value 
we have
$F_\beta(\0,\0) = 1$.
The most commonly used instantiation is the $F_1$-measure, which weighs recall and precision equally; other commonly used variants include the $F_2$-measure, which weighs recall more heavily than precision, and the $F_{0.5}$-measure, which weighs precision more heavily than recall.

\textbf{Learning goal.}
Assuming that training examples are drawn IID from some underlying probability distribution $D$ on $\X\times\{0,1\}^s$, it is natural then to measure the quality of a multi-label classifier $\h:\X\>\{0,1\}^s$ by its \emph{$F_\beta$-generalization accuracy}:\footnote{Note that our focus is on \emph{instance-averaged} $F_\beta$ performance \cite{Zhang2014}.}
\vspace{-6pt}
\[
\acc^{F_\beta}_D[\,\h\,]
	~ = ~
	\E_{(x,\y)\sim D} [\, F_\beta(\y,\h(x)) \,]
	\,.
\vspace{-2pt}	
\]
The \emph{Bayes $F_\beta$-accuracy} is then the highest possible value of the $F_\beta$-generalization accuracy for $D$:
\vspace{-4pt}
\[
\acc^{F_\beta,*}_D 
	~ = ~
	\sup_{\h:\X\>\{0,1\}^s} \acc^{F_\beta}_D[\,\h\,]
	\,.
\vspace{-2pt}	
\]
The \emph{$F_\beta$-regret} of a multi-label classifier $\h$ is then the difference between the Bayes $F_\beta$-accuracy and the $F_\beta$-accuracy of $\h$:
\vspace{-4pt}
\[
\regret^{F_\beta}_D[\,\h\,]
	~ = ~
	\acc^{F_\beta,*}_D 
	-
	\acc^{F_\beta}_D[\,\h\,]
	\,.
\vspace{-2pt}	
\] 
Our goal will be to design \emph{consistent} algorithms for the $F_\beta$-measure, i.e.\ algorithms whose $F_\beta$-regret converges (in probability) to zero as the number of training examples increases. In particular, since we cannot maximize the (discrete) $F_\beta$-measure directly, we would like to design consistent algorithms that maximize a concave surrogate performance measure -- or equivalently, minimize a convex surrogate loss -- instead. For this, we will turn to the theory of convex calibrated surrogates.

\vspace{-2pt}
%-----------------
\subsection{Convex Calibrated Surrogates for Multiclass Problems}
%Prior Result of \citet{Ramaswamy+14}}
\vspace{-2pt}

Here we review the theory of convex calibrated surrogates for multiclass 
classification problems, and in particular, the result of \citet{Ramaswamy+14} for low-rank multiclass loss matrices that we will use in our work. 
%loss matrices.
We will apply the theory to the multi-label $F_\beta$-measure in \Sec{sec:surrogates}. 

\textbf{Multiclass classification.}
Consider a standard multiclass (not multi-label) learning problem with instance space $\X$ and label space $\Y=[n]$
%:=\{1,\ldots,n\}$ 
(i.e., $n$ classes).
Let $\L\in\R_+^{n\times n}$ be a loss matrix whose $(y,\hat{y})$-th entry $\ell_{y,\hat{y}} = \ell(y,\hat{y})$ (for each $y,\hat{y}\in[n]$) specifies the loss incurred on predicting $\hat{y}$ when the true label is $y$ (the 0-1 loss $\L^{\zo}$ is a special case with $\ell^\zo_{y,\hat{y}} = \1(\hat{y}\neq y)$). Then, given a training sample $S=((x_1,y_1),\ldots,(x_m,y_m))\in(\X\times\Y)^m$ with examples drawn IID from some underlying probability distribution $D$ on $\X\times\Y$, the performance of a classifier $h:\X\>\Y$ is measured by its $\L$-generalization error $\er_D^\L[h]=\E_{(x,y)\sim D}[\, \ell_{y,h(x)} \,]$, or its $\L$-regret $\regret_D^\L[h] = \er_D^\L[h] - \er_D^{\L,*}$, where $\er_D^{\L,*} = \inf_{h:\X\>\Y} \er_D^\L[h]$ is the Bayes $\L$-error for $D$. 
A learning algorithm that maps training samples $S$ to classifiers $h_S$ is said to be (universally) \emph{$\L$-consistent} if for all $D$ and for $S\sim D^m$, $\regret_D^\L[h_S] \ip 0$ as $m\>\infty$.

\textbf{Surrogate risk minimization and calibrated surrogates.}
Since minimizing the discrete loss $\L$ directly is computationally hard, a common algorithmic framework is 
%that of \emph{surrogate risk minimization}, where one minimizes a convex surrogate loss instead. In particular, for a surrogate loss $\psi:
to minimize a surrogate loss $\psi:[n]\times\R^d\>\R_+$ for some suitable $d\in\Z_+$.
In particular, given a multiclass training sample $S$ as above, one learns a $d$-dimensional `scoring' function $\f_S:\X\>\R^d$ by solving 
\vspace{-2pt}
\[
\textstyle{\min_{\f} \sum_{i=1}^m \psi(y_i,\f(x_i))}
\vspace{-2pt}
\]
over a suitably rich class of functions $\f:\X\>\R^d$; and then returns $h_S = \decode \circ \f_S$ for some suitable mapping $\decode:\R^d\>[n]$.
% (for example, for multiclass 0-1 classification, where $k=n$ and $\ell_\zo(y,t) = \1(t\neq y)$, many common algorithms such as those considered by \cite{Zhang04b} and \cite{TewariBa07} learn a function $\f_m:\X\>\R^n$ and then return a classifier $h_m = \argmax\circ\f_m$).
In practice, the surrogate $\psi$ is often chosen to be convex in its second argument to enable efficient minimization.
It is known that if the minimization is performed over a universal function class (with suitable regularization), then the resulting algorithm is universally \emph{$\psi$-consistent},
% w.r.t.\ $D$, i.e.\ that the \emph{$\psi$-generalization error} of $\f_S$ w.r.t.\ $D$, defined for a function $\f:\X\>\R^d$ as $\er_D^\psi[\f] = \E_{(x,y)\sim D}[\psi(y,\f(x))]$, converges to the optimal:
i.e.\ that the $\psi$-regret converges to zero: $\regret_D^\psi[\f_S] = \er_D^\psi[\f_S] - \er_D^{\psi,*} \ip 0$ as $m\>\infty$ (where $\er_D^\psi[\f] = \E_{(x,y)\sim D} [\, \psi(y,\f(x)) \,]$ is the $\psi$-generalization error of $\f$ and $\er_D^{\psi,*} = \inf_{\f:\X\>\R^d} \er_D^\psi[\f]$ is the Bayes $\psi$-error).
The surrogate $\psi$, together with the mapping $\decode$, is said to be \emph{$\L$-calibrated} if this also implies $\L$-consistency, i.e.\ if 
%this also implies $\regret_D^\L[\decode \circ \f_S] \ip 0$ as $m\>\infty$.
\vspace{-4pt}
\[
\regret_D^\psi[\f_S] \ip 0
	\implies 
	\regret_D^\L[\decode \circ \f_S] \ip 0
	\,.
\vspace{-1pt}
\]
Thus, given a target loss $\L$, the task of designing an $\L$-consistent algorithm reduces to designing a convex $\L$-calibrated surrogate-mapping pair $(\psi,\decode)$; the resulting surrogate risk minimization algorithm (implemented in a universal function class with suitable regularization) is then universally $\L$-consistent.

%\(
%^\er_D^\psi[\f_S] \ip \inf_{\f:\X\>\R^d} \er_D^\psi[\f]
%	\,.
%\)
%There has been much work over the last several years on understanding when $\psi$-consistency (of $\f_m$) also implies $\ell$-consistency (of $h_m$), and how to design surrogates satisfying this property; in particular, this has led to the
% important notion 
%study of surrogates that are \emph{calibrated} with respect to the target loss $\ell$ \citep{Bartlett+06,Zhang04a,Zhang04b,TewariBa07,Steinwart07,RamaswamyAg12}.

\textbf{Result of \citet{Ramaswamy+14} for low-rank loss matrices.}
The result of \citet{Ramaswamy+14} effectively decomposes multiclass problems into a set of binary CPE problems; to describe the result, we will need the following definition for \emph{binary} losses:

\begin{defn}[Strictly proper composite binary losses \cite{ReidWi10}]
A binary loss $\phi:\{\pm1\}\times\R\>\R_+$ is \emph{strictly proper composite with underlying (invertible) link function $\gamma:[0,1]\>\R$} if for all $q\in[0,1]$ and $u\neq\gamma(q)\in\R$:
\vspace{-5pt}
\[
\E_{y\sim {\text{Bin}}^{\pm1}(q)}\Big[ \phi(y,u) - \phi(y,\gamma(q)) \Big] ~ > ~ 0
	\,,
%\vspace{-2pt}
\]
where $y\sim {\text{Bin}}^{\pm1}(q)$ denotes a $\{\pm1\}$-valued random variable that takes value $+1$ with probability $q$ and value $-1$ with probability $1-q$.
\end{defn}

Intuitively, minimizing a strictly proper composite binary loss allows one to recover accurate class probability estimates for binary CPE problems: the learned real-valued score is simply inverted via $\gamma^{-1}$ \cite{ReidWi10}.

%For multiclass loss matrices $\L$ of rank $r$, \citet{Ramaswamy+14} give 
We can now state the result of \citet{Ramaswamy+14}, which for multiclass loss matrices $\L$ of rank $r$, gives
a family of $r$-dimensional convex $\L$-calibrated surrogates defined in terms of strictly proper composite binary losses as follows 
(result specialized here to the case of square loss matrices, and stated with a small change in normalization):
%(result stated here for square loss matrices, and with a small change in normalization):

\begin{thm}[\citet{Ramaswamy+14}]
\label{thm:Ramaswamy+14}
Let $\L\in\R_+^{n\times n}$ be a rank-$r$ multiclass loss matrix, with 
%$\L=\A^\top\B$ for some $\A,\B\in\R_+^{r\times n}$.
$\ell_{y,\hat{y}} = \a_y^\top \b_{\hat{y}}$ for some $\a_1,\ldots,\a_n,\b_1,\ldots,\b_n \in \R^r$.
Let $\phi:\{\pm1\}\times\R\>\R_+$ be any strictly proper composite binary loss, with underlying link function $\gamma:[0,1]\>\R$. 
Define a multiclass surrogate $\psi:[n]\times\R^r\>\R_+$ and mapping $\decode:\R^r\>[n]$ as follows:
\vspace{-6pt}
\[
\psi(y,\u) ~ = ~ \sum_{j=1}^r \Big( \tilde{a}_{yj} \phi(+1,u_j) + (1-\tilde{a}_{yj}) \phi(-1,u_j) \Big)
\]
\vspace{-12pt}
\[
\decode(\u) ~ \in ~ \argmin_{\hat{y}\in[n]} \sum_{j=1}^r \tilde{b}_{\hat{y}j} \gamma^{-1}(u_j) + c_{\hat{y}}
	\,, 
\vspace{-2pt}
\]
where 
\vspace{-8pt}
\begin{eqnarray*}
\tilde{a}_{yj} & = & {\frac{a_{yj}-a_{\min}}{a_{\max}-a_{\min}}} ~~~ (\in[0,1]) \\
\tilde{b}_{\hat{y}j} & = & (a_{\max} - a_{\min}) \cdot b_{\hat{y}j} \\
c_{\hat{y}} & = & a_{\min} \, \textstyle{\sum_{j=1}^r} b_{\hat{y}j} \\
a_{\min} & = & \min_{y,j} a_{yj} \\
a_{\max} & = & \max_{y,j} a_{yj} \,.
%Z & = & \max_{y,j} (a_{yj}-a_{\min}) \,.
\end{eqnarray*}
Then $(\psi,\decode)$ is $\L$-calibrated.
\end{thm}

The above result effectively decomposes the multiclass problem into $r$ binary CPE problems, where the labels for these CPE problems can themselves be given as probabilities in $[0,1]$ rather than binary values (see \citet{Ramaswamy+14} for details). 
%The multiclass surrogate $\psi$ provided by the above result makes use of a strictly proper composite binary loss $\phi$ for these binary CPE problems; such losses were studied in detail by \citet{ReidWi10}.
For our purposes, we will use the standard binary logistic loss for the binary CPE problems, which is known to be strictly proper composite (see \Sec{sec:surrogates} below for more details).
% (an example will be provided in the next section). 
%For our purposes, we will use the binary logistic loss $\phi_{\log}:\{\pm1\}\times\R\>\R_+$ given by 
%\(
%\phi_{\log}(y,u) = \log(1+e^{-yu})
%	\,;
%\)
%for the binary CPE problems; 
%this is known to be strictly proper composite, with underlying link function $\gamma_{\log}:[0,1]\>\R$ given by the logit link  $\gamma_{\log}(p) = \ln\big(\frac{p}{1-p}\big)$ \cite{ReidWi10}.

\vspace{-4pt}
%========== SECTION 4 ===========
\section{Convex Calibrated Surrogates for $F_\beta$}
\label{sec:surrogates}
\vspace{-2pt}

In order to construct convex calibrated surrogates -- and corresponding surrogate risk minimization algorithms -- for the multi-label $F_\beta$-measure, we will start by viewing the multi-label learning problem as a giant multiclass classification problem with $n=2^s$ classes (this is only for the purpose of analysis and derivation of the surrogates; as we will see, the actual algorithms we will obtain will require learning only $O(s^2)$ real-valued score functions).
To this end, let us define the $F_\beta$-loss matrix $\L^{F_\beta}\in\R_+^{\{0,1\}^s \times \{0,1\}^s}$ as follows:
\vspace{-2pt}
\[
\ell^{F_\beta}_{\y,\hat{\y}} ~ = ~ 1 - F_\beta(\y,\hat{\y})
	\,.
\]

\textbf{$\L^{F_\beta}$ has low rank.}
We show here that (a slightly shifted version of) the above loss matrix has rank at most $s^2+1$.

\begin{prop}
\label{prop:rank-F}
$\rank(\L^{F_\beta}-1) \leq s^2 + 1$.
\end{prop}
\vspace{-10pt}
\begin{proof}
%From \Eqn{eqn:F-beta}, w
We have,
\[
\ell^{F_\beta}_{\y,\hat{\y}} - 1 
	~ = ~
	- F_\beta(\y,\hat{\y})
	~ = ~
	-\frac{(1+\beta^2) \sum_{j=1}^s y_j \hat{y}_j}{\beta^2 \|\y\|_1 + \|\hat{\y}\|_1}
	\,.
\]
Stratifying over the $s+1$ different values of $\|\y\|_1 \in\{0,1\ldots,s\}$, we can write this as
\begin{eqnarray*}
\ell^{F_\beta}_{\y,\hat{\y}} - 1 
	& = &
	- \, \1(\|\y\|_1=0) \cdot \1(\|\hat{\y}\|_1=0) 
\\
	& &
	- \, \sum_{k=1}^s \1(\|\y\|_1=k) \cdot \frac{(1+\beta^2) \sum_{j=1}^s y_j \hat{y}_j}{\beta^2 k + \|\hat{\y}\|_1}
%\\
%	& = &
%	- \, \1(\|\y\|_1=0) \cdot \1(\|\hat{\y}\|_1=0) 
%\\
%	& &
%	- \, \sum_{k=1}^s \sum_{j=1}^s (1+\beta^2)  \1(\|\y\|_1=k) y_j \cdot \frac{\hat{y}_j}{\beta^2 k + \|\hat{\y}\|_1}
\\
	& = &
	a_{\y,0} \cdot b_{\hat{\y},0} + \sum_{j=1}^s \sum_{k=1}^s a_{\y,jk} \cdot b_{\hat{\y},jk}
	\,,
\end{eqnarray*}
where
\begin{eqnarray}
a_{\y,0} & = & \1(\|\y\|_1=0)
\label{eqn:a}
\\
b_{\hat{\y},0} & = & - \1(\|\hat{\y}\|_1=0)
\\
a_{\y,jk} & = & \1(\|\y\|_1=k) \cdot y_j 
\\
b_{\hat{\y},jk} & = & - \frac{(1+\beta^2) \cdot \hat{y}_j}{\beta^2 k + \|\hat{\y}\|_1}
\label{eqn:b}
	\,.
\end{eqnarray}
This proves the claim.
\end{proof}

\textbf{$\L^{F_\beta}$-calibrated surrogates.}
Given the above result, we can now apply \Thm{thm:Ramaswamy+14} to construct a family of $(s^2+1)$-dimensional convex calibrated surrogate losses for $\L^{F_\beta}$.\footnote{Note that minimizing the $\L^{F_\beta}$-generalization error is equivalent to minimizing the $(\L^{F_\beta}-1)$-generalization error, and therefore a calibrated surrogate for $\L^{F_\beta}-1$ is also calibrated for $\L^{F_\beta}$.}
Specifically, starting with any strictly proper composite binary loss $\phi:\{\pm1\}\times\R\>\R_+$ with underlying link function $\gamma:[0,1]\>\R$, we
define a multiclass surrogate $\psi:\{0,1\}^s\times\R^{s^2+1}\>\R_+$ and mapping $\decode:\R^{s^2+1}\>\{0,1\}^s$ as follows (where we denote $\u = \big( u_0,(u_{jk})_{j,k=1}^s \big)^\top \in \R^{s^2+1}$):
\begin{eqnarray}
\lefteqn{\psi(\y,\u)}
\nonumber
\\
	& \hspace{-12pt} = & \hspace{-8pt} 
	a_{\y,0} \cdot \phi(+1,u_0) + (1-a_{\y,0}) \cdot \phi(-1,u_0) 
\nonumber
\\
	& \hspace{-12pt} & \hspace{-10pt} 
	+ \sum_{j=1}^s \sum_{k=1}^s a_{\y,jk} \cdot \phi(+1,u_{jk}) + (1-a_{\y,jk}) \cdot \phi(-1,u_{jk}) 
\nonumber
\\
\label{eqn:psi}
\\%[4pt]
\lefteqn{\decode(\u)}
\nonumber
\\[-2pt]
	& \hspace{-12pt}  \in & \hspace{-8pt} 
	\argmin_{\hat{\y}\in\{0,1\}^s} \,\, b_{\hat{\y},0}\cdot \gamma^{-1}(u_0) + \sum_{j=1}^s \sum_{k=1}^s b_{\hat{\y},jk} \cdot \gamma^{-1}(u_{jk}) 
	\,,
\nonumber
\\[-2pt]
\label{eqn:decode}
\end{eqnarray}
where $a_{\y,0}, a_{\y,jk}, b_{\hat{\y},0}, b_{\hat{\y},jk}$  are as defined in \Eqs{eqn:a}{eqn:b}.
Then, by \Thm{thm:Ramaswamy+14} and the proof of \Prop{prop:rank-F}, it follows that $(\psi,\decode)$ is $\L^{F_\beta}$-calibrated.\footnote{Note that when applying \Thm{thm:Ramaswamy+14} here, we have $a_{\min}=0$ and $a_{\max}=1$, and therefore $\tilde{\a}_\y = \a_\y$, $\tilde{\b}_{\hat{\y}} = \b_{\hat{\y}}$, and $c_{\hat{\y}} = 0$.}
%Consequently, any surrogate risk minimization algorithm using such a $(\psi,\decode)$ pair (and implemented in a universal function class with suitable regularization) is consistent for the $F_\beta$-measure. 
Therefore, the resulting $(\psi,\decode)$-based surrogate risk minimization algorithm,
when implemented in a universal function class (with suitable regularization), is consistent for the $F_\beta$-measure. 
The algorithm is summarized in \Algo{algo}. Note that since $a_{\y,0}, a_{\y,jk} \in \{0,1\}$, in this case minimizing the surrogate risk above amounts to solving $s^2+1$ binary CPE problems with standard binary (non-probabilistic) labels.   

\textbf{Choice of strictly proper composite binary loss $\phi$.}
%In particular, a
As a specific instantiation, in our experiments, we will make use of the binary logistic loss $\phi_{\log}:\{\pm1\}\times\R\>\R_+$ given by 
\begin{eqnarray}
\phi_{\log}(y,u) & = & \ln(1+e^{-yu})
%	\,;
\label{eqn:logistic}
\end{eqnarray}
as the binary loss above; this is known to be strictly proper composite \cite{ReidWi10}, with underlying logit link function $\gamma_{\log}:[0,1]\>\R$ given by 
%the logit link 
\begin{eqnarray}
\gamma_{\log}(p) & = & \ln\Big(\frac{p}{1-p}\Big)
	\,.
\label{eqn:logit}
\end{eqnarray}
%The resulting surrogate risk minimization algorithm for the multi-label $F_\beta$-measure is summarized in \Algo{algo}. 

\begin{algorithm}[t]
\caption{Surrogate risk minimization algorithm for multi-label $F_\beta$-measure}
\begin{algorithmic}[1]
\STATE \textbf{Input:} 
	Training sample %\\
%	~~~~~~~~~~ 
	$S=((x_1,\y_1),\ldots,(x_m,\y_m))$ $\in(\X\times\{0,1\}^s)^m$
\STATE \textbf{Parameters:} 
	(1) Strictly proper composite binary CPE loss $\phi:\{\pm1\}\times\R\>\R_+$; 
	(2) Class $\F$ of functions $\f:\X\>\R^{s^2+1}$
\STATE Find $\f_S\in\argmin_{\f\in\F} \sum_{i=1}^m \psi(\y_i,\f(x_i))$, where $\psi$ is as defined in \Eqn{eqn:psi}
%	$\hat{\f} \in \argmin_{\f\in\F} \sum_{i=1}^m \psi(y_i,\f(x_i))$ (see \Eqn{eqn:psi})
%\STATE Find $\f_S\in\F$ that minimizes
%\begin{eqnarray*}
%\hat{\f} 
%	& \hspace{-6pt} \in \hspace{-6pt} & 
%\hspace{-12pt} & \hspace{-12pt} &	
%	\sum_{i=1}^m \bigg( 
%	\1(\|\y_i\|_1 = 0)\cdot \ln(1+e^{-f_0(x_i)}) +  
%\\[-2pt]
%\hspace{-12pt} & \hspace{-12pt} &	
%	~~~~~~~~~ \big(1- \1(\|\y_i\|_1 = 0) \big) \cdot \ln(1+e^{f_0(x_i)}) + 
%\\
%\hspace{-12pt} & \hspace{-12pt} &
%	~~~~~~~~ \sum_{j=1}^s \sum_{k=1}^s 
%	y_j \, \1(\|\y_i\|_1 = k) \cdot \ln(1+e^{-f_{jk}(x_i)}) +  
%\\
%\hspace{-12pt} & \hspace{-12pt} &
%	~~~~~~~~ \sum_{j=1}^s \sum_{k=1}^s 
%	\big(1- y_j \, \1(\|\y_i\|_1 = k) \big) \cdot \ln(1+e^{-f_{jk}(x_i)}) 
%	\bigg) 
%\end{eqnarray*}
\STATE \textbf{Output:} Multi-label classifier $\h_S = \decode \circ \f_S$, where $\decode$ is as defined in \Eqn{eqn:decode} (see Appendix for efficient implementation of $\decode$)
\end{algorithmic}
\label{algo}
\end{algorithm}

\textbf{Implementation of `$\decode$' mapping.}
The mapping $\decode:\R^{s^2+1}\>\{0,1\}^s$ above can be implemented in $O(s^3)$ time using a procedure due to \citet{Dembczynski2011}; details are provided in the Appendix for completeness. In particular, \citet{Dembczynski2011} show that if one knows the true conditional MLC distribution $p(\y|x)$, then one can use $s^2+1$ statistics of this distribution to construct a Bayes optimal classifier for the $F_\beta$-measure; they then provide a procedure to perform this computation in $O(s^3)$ time. 
%In our setting, we apply the same procedure to the $s^2+1$ quantities estimated by the 
As we discuss in greater detail in \Sec{sec:relationship}, our surrogate loss $\psi$ can be viewed as computing estimates of the same $s^2+1$ statistics from the training sample $S$, and therefore our algorithm, which applies the `decoding' procedure of \citet{Dembczynski2011} to these estimated quantities, can be viewed as effectively learning a form of `plug-in' multi-label classifier for the $F_\beta$-measure.

%\vspace{-8pt}
%========== SECTION 5 ===========
\section{Regret Transfer Bound}
\label{sec:regret-bound}
%\vspace{-4pt}

Above, we constructed a family of $\L^{F_\beta}$-calibrated surrogate-mapping pairs $(\psi,\decode)$ (\Eqs{eqn:psi}{eqn:decode}), yielding a family of surrogate risk minimization algorithms for the $F_\beta$-measure (\Algo{algo}).
We now give a quantitative regret transfer bound 
%relating the target $F_\beta$-regret to the surrogate $\psi$-regret. 
showing that any guarantees on the surrogate $\psi$-regret also translate to guarantees on the  target $F_\beta$-regret.
Specifically, the surrogate loss $\psi$ was defined in terms of a constituent strictly proper composite binary loss $\phi:\{\pm1\}\times\R\>\R_+$. We show that if the binary loss $\phi$ is \emph{strongly} proper composite (a relatively mild condition satisfied by several common strictly proper composite binary losses, including the logistic loss), then for all models $\f:\X\>\R^{s^2+1}$, we can upper bound $\regret_D^{F_\beta}[\decode \circ \f]$, the target $F_\beta$-regret of the multi-label classifier given by $\h(x) = \decode(\f(x))$, in terms of $\regret_D^\psi[\f]$, the surrogate regret of $\f$.
In order to prove the regret transfer bound, we 
%first give a quick overview of strongly proper composite binary losses (for more details, see \cite{Agarwal14}).
will need the following definition:

%\textbf{Strongly proper composite binary losses.}
\begin{defn}[Strongly proper composite binary losses \cite{Agarwal14}]
Let $\lambda > 0$. A binary loss $\phi:\{\pm1\}\times\R\>\R_+$ is said to be \emph{$\lambda$-strongly proper composite with underlying (invertible) link function $\gamma:[0,1]\>\R$} if for all $q\in[0,1]$, $u\in\R$:
\vspace{-4pt}
\[
\E_{y\sim {\text{Bin}}^{\pm1}(q)}\Big[ \phi(y,u) - \phi(y,\gamma(q)) \Big] ~ \geq ~ \frac{\lambda}{2} \Big( \gamma^{-1}(u) - q \Big)^2
	\,.
\]
%where $y\sim {\text{Bin}}^{\pm1}(q)$ denotes a $\{\pm1\}$-valued random variable that takes value $+1$ with probability $q$ and value $-1$ with probability $1-q$.
\end{defn}

We note that 
the logistic loss (\Eqn{eqn:logistic}) is known to be 4-strongly proper composite with underlying link given by the logit link (\Eqn{eqn:logit}) \cite{Agarwal14}.

\textbf{Additional notation.}
%In order to 
%state and 
To prove our regret transfer bound, we will also need some additional notation. 
In particular, 
for each $\y,\hat{\y}\in \{0,1\}^s$, we will define the vectors 
%$\a_\y\in\{0,1\}^{s^2+1}$, $\b_{\hat{\y}}\in\R^{s^2+1}$ as follows:
\begin{eqnarray}
\a_\y & = & 
	\left( \begin{array}{c}
		a_{\y,0} \\
		a_{\y,11} \\
		\vdots \\
		a_{\y,ss} 
	\end{array} \right)
	\in \{0,1\}^{s^2+1}
\label{eqn:a-y}
\\
\b_{\hat{\y}} & = & 
	\left( \begin{array}{c}
		b_{\hat{\y},0} \\
		b_{\hat{\y},11} \\
		\vdots \\
		b_{\hat{\y},ss} 
	\end{array} \right)
	\in \R^{s^2+1}
	\,,
\label{eqn:b-y-hat}
\end{eqnarray}
where $a_{\y,0}, a_{\y,jk}, b_{\hat{\y},0}, b_{\hat{\y},jk}$ are as defined in \Eqs{eqn:a}{eqn:b}.
Moreover, for each $x\in\X$, we will define 
\begin{eqnarray}
%q_0(x) & = & \sum_{\y\in\{0,1\}^s} p(\y|x) a_{\y,0} ~~ \in[0,1]
\q(x) 
	& \hspace{-7pt} = & \hspace{-7pt} 
	\E_{\y|x}[ \a_\y ] 
	= \hspace{-4pt} 
	\sum_{\y\in\{0,1\}^s} \hspace{-4pt} p(\y|x) \cdot \a_\y ~ \in[0,1]^{s^2+1}
	\,. ~~~~~~~~
\label{eqn:q}
\end{eqnarray}
Intuitively, the elements $q_0(x), (q_{jk}(x))_{j,k=1}^s$ of $\q(x)$ are the `class probability functions' corresponding to the $s^2+1$ binary CPE problems effectively created by the surrogate loss $\psi$ defined in \Eqn{eqn:psi}. 
The function $\f_S:\X\>\R^{s^2+1}$ learned by minimizing $\psi$ will be such that $\gamma^{-1}(\f_S(x))$ will serve as an estimate of $\q(x)$. 

\textbf{Regret transfer bound.} 
We are now ready to state and prove the following regret transfer bound for the family of surrogate losses defined in the previous section:

\begin{thm}
\label{thm:regret-bound}
Let $\phi:\{\pm1\}\times\R\>\R_+$ be a $\lambda$-strongly proper composite binary loss with underlying link function $\gamma:[0,1]\>\R$. 
Let ($\psi,\decode$) be defined as in \Eqs{eqn:psi}{eqn:decode}.
Then for all probability distributions $D$ on $\X\times\{0,1\}^s$ and all $\f:\X\>\R^{s^2+1}$, we have
\[
\regret_D^{F_\beta}[\decode \circ \f] 
	~ \leq ~
%	2 \max_{\hat{\y}} \, \big\| \b_{\hat{\y}} \big\|_2 
%	\cdot \sqrt{\frac{2}{\lambda}\regret_D^\psi[\f]}
	\frac{1+\beta^2}{\beta} 
	\sqrt{\frac{2(\ln s+1)}{\lambda} \cdot \regret_D^\psi[\f]}
	\,.
\]
\end{thm}
\begin{proof}
We have,
\begin{eqnarray}
\lefteqn{
\regret_D^{F_\beta}[\decode \circ \f] 
}
\nonumber
\\	
	& \hspace{-8pt} = &
	\hspace{-8pt} 
	\E_x\Big[ \sum_{\y} p(\y|x) \cdot \Big( \ell^{F_\beta}_{\y,\decode(\f(x))} - \min_{\hat{\y}} \ell^{F_\beta}_{\y,\hat{\y}} \Big) \Big]
\nonumber
\\	
	& \hspace{-8pt} = &
	\hspace{-8pt} 
	\E_x\Big[ \sum_{\y} p(\y|x) \cdot \Big( \a_\y^\top \b_{\decode(\f(x))} - \min_{\hat{\y}} \a_\y^\top \b_{\hat{\y}} \Big) \Big]
\nonumber
\\	
	& \hspace{-8pt} = &
	\hspace{-8pt} 
	\E_x\Big[ \q(x)^\top \b_{\decode(\f(x))} - \min_{\hat{\y}} \q(x)^\top \b_{\hat{\y}}  \Big]
\nonumber
\\	
	& \hspace{-8pt} = &
	\hspace{-8pt} 
	\E_x\Big[ \max_{\hat{\y}} \,\, \q(x)^\top \Big( \b_{\decode(\f(x))} - \b_{\hat{\y}} \Big) \Big]	
\nonumber
\\	
	& \hspace{-8pt} \leq &
	\hspace{-8pt} 
	\E_x\Big[ \max_{\hat{\y}} \,\, \Big( \q(x) - \gamma^{-1}(\f(x)) \Big) ^\top \Big( \b_{\decode(\f(x))} - \b_{\hat{\y}} \Big) \Big]
\nonumber
\\
	& &
	~~~~~~~~~~\text{(since by the definition of $\decode$,} 
\nonumber
\\
	& & 
	~~~~~~~~~~\text{$-\gamma^{-1}(\f(x))^\top \Big( \b_{\decode(\f(x))} - \b_{\hat{\y}} \Big) \geq 0 ~~\forall ~\hat{\y}$)}
\nonumber
\\
	& \hspace{-8pt} \leq &
	\hspace{-8pt} 
	\E_x\Big[ \big\| \q(x) - \gamma^{-1}(\f(x)) \big\|_2  \cdot  \max_{\hat{\y}} \, \big\| \b_{\decode(\f(x))} - \b_{\hat{\y}} \big\|_2 \Big]
\nonumber
\\
	& &
	~~~~~~~~~~\text{(by the Cauchy-Schwarz inequality)}
\nonumber
\\
	& \hspace{-8pt} \leq &
	\hspace{-8pt} 
	2 \max_{\hat{\y}} \, \big\| \b_{\hat{\y}} \big\|_2 \cdot \E_x\Big[ \big\| \q(x) - \gamma^{-1}(\f(x)) \big\|_2  \Big]
	\,.
\label{eqn:proof-1}
\end{eqnarray}
%
%Moreover, 
Now,
since $\phi$ is $\lambda$-strongly proper composite with link function $\gamma$, we have %for each $x\in\X$,
\begin{eqnarray}
%\E_x\Big[ \big\| \q(x) - \gamma^{-1}(\f(x)) \big\|_2  \Big]
%	& \leq & 
%	\sqrt{ \frac{2}{\lambda} \regret_D^\psi[\f] }
%	\,. 
%	~~~~~
\lefteqn{
\E_x\Big[ \big\| \q(x) - \gamma^{-1}(\f(x)) \big\|_2^2  \Big]
%\big\| \q(x) - \gamma^{-1}(\f(x)) \big\|_2^2
}
\nonumber
\\
	& \hspace{-8pt} = & 
	\hspace{-8pt} 
	\E_x\Big[ 	
		\Big( q_0(x) - \gamma^{-1}(f_0(x)) \Big)^2 + 
\nonumber
\\
	& & 
		~~~~~~~~ \sum_{j=1}^s \sum_{k=1}^s \Big( q_{jk}(x) - \gamma^{-1}(f_{jk}(x)) \Big)^2 
	\Big]
\nonumber
\\
	& \hspace{-8pt} \leq & 
	\hspace{-8pt}
	\frac{2}{\lambda} \E_x\Big[ 	
		\E_{y\sim \text{Bin}^{\pm1}(q_0(x))}\Big[ \phi(y,f_0(x)) - \phi(y,\gamma(q_0(x)) \Big] +
%	\Big]
%	\ldots
\nonumber
\\
	& & 
	\hspace{-8pt}
		\sum_{j=1}^s \sum_{k=1}^s \E_{y\sim \text{Bin}^{\pm1}(q_{jk}(x))}\Big[ \phi(y,f_{jk}(x)) - \phi(y,\gamma(q_{jk}(x)) \Big]
	\Big]
\nonumber
\\
	& &
	~~~~~ \text{(by $\lambda$-strong proper compositeness of $\phi$)}
\nonumber
\\
	& \hspace{-8pt} = & 
	\hspace{-8pt}
	\frac{2}{\lambda} \E_x\Big[ 	
		\E_{\y|x} \Big[ \psi(\y,\f(x)) - \inf_{\u\in\R^{s^2+1}} \psi(\y,\u) \Big] 	
	\Big]
\nonumber
\\
	& \hspace{-8pt} = & 
	\hspace{-8pt}
%	\frac{2}{\lambda} \regret_D^\psi[\f] 
	\frac{2}{\lambda} \regret_D^\psi[\f] 
	\,. 
	~~~~~
\label{eqn:proof-2}
\end{eqnarray}
%Moreover, it can be verified that 
%\begin{eqnarray}
%\max_{\hat{\y}} \, \big\| \b_{\hat{\y}} \big\|_2 
%	& \leq &
%	s
%	\,.
%\label{eqn:proof-3}
%\end{eqnarray}
Moreover, we have
\[
\|\b_\0\| = 1
	\,,
\]
and for $\hat{\y}\neq \0$, we have
\vspace{-7pt}
\begin{eqnarray*}
\|\b_{\hat{\y}}\|_2^2  
	& = & 
	(1+\beta^2)^2 \sum_{k=1}^s g_k(\|\hat{\y}\|_1)
	\,,
\\[-10pt]
\end{eqnarray*}
where 
\vspace{-8pt}
\[
g_k(t) = \frac{t}{(\beta^2 k + t)^2}
	\,.
\]
It can be verified that $g_k(t)$ is maximized at $t^* = \beta^2 k$, yielding for each $\hat{\y}\neq \0$,
\vspace{-4pt}
\begin{eqnarray*}
\|\b_{\hat{\y}}\|_2^2 
	& \leq & 
	(1+\beta^2)^2 \sum_{k=1}^s g_k(\beta^2 k)
\nonumber
\\
	& = & 
	(1+\beta^2)^2 \sum_{k=1}^s \frac{1}{4\beta^2 k}
\nonumber
\\
	& \leq &
	\frac{(1+\beta^2)^2}{4\beta^2} (\ln s + 1) 
\nonumber
\\
	& &
	~~~~~~~~~ \textrm{(since $\textstyle{\sum_{k=1}^s} \frac{1}{k} \leq \ln s + 1$)}
	\,.
\nonumber
\\[-20pt]
\nonumber
\end{eqnarray*}
This gives
\vspace{-12pt}
\begin{eqnarray*}
\max_{\hat{\y}} \|\b_{\hat{\y}}\|_2 
	& \leq &
	\frac{(1+\beta^2)}{2\beta} \sqrt{\ln s + 1}
	\,.
\\[-18pt]
\label{eqn:proof-3}
\end{eqnarray*}
Combining \Eqs{eqn:proof-1}{eqn:proof-3} and applying Jensen's inequality (to the convex function $g(z)=z^2$) proves the claim.
\end{proof}

\textbf{Remark.}
We note that \Thm{thm:regret-bound} gives a self-contained proof that the surrogate-mapping pair $(\psi,\decode)$ defined in \Eqs{eqn:psi}{eqn:decode} is $\L^{F_\beta}$-calibrated, since the result implies that for any sequence of models $\f_S$ learned from training samples $S\sim D^m$ of increasing size $m$, 
\vspace{-4pt}
\[
\regret_D^\psi[\f_S] \ip 0
	\implies 
	\regret_D^{F_\beta}[\decode \circ \f_S] \ip 0
	\,.
\vspace{-4pt}
\]
Nevertheless, since the design of our surrogate-mapping pair $(\psi,\decode)$ was based on the work of \citet{Ramaswamy+14}, we chose to present their calibration result (\Thm{thm:Ramaswamy+14}) first. We also note that, while we have stated the above regret transfer bound for the $F_\beta$-measure, a similar bound also applies more generally to all multiclass problems with low-rank matrices as considered in \Thm{thm:Ramaswamy+14}, thus yielding a stronger (quantitative) result than \Thm{thm:Ramaswamy+14} \cite{Ramaswamy15-thesis}.

\vspace{-4pt}
%========== SECTION 6 ===========
\section{Relationship with Plug-in Algorithm of \citet{Dembczynski2013}}
\label{sec:relationship}
\vspace{-2pt}

The plug-in algorithm of \citet{Dembczynski2013}, termed \emph{exact $F$-measure plug-in} (EFP), estimates the following statistics of the conditional label distribution $p(\y|x)$:
\vspace{-4pt}
\begin{eqnarray*}
\lefteqn{\P(\|\y\|_1=0 \,|\, x)} 
~~~~~~~~~~~~~~~~~~~~~~~~~~~~~~~~~~~~~~~~~~~~~~~~~~~~~~~~~~~~~~~~~~~~~~~~~~~ & &
\\
\lefteqn{\P(\|\y\|_1=k, y_j=1 \,|\, x) \,, ~ \P(y_j=0 \,|\, x) \,, ~~ j,k\in[s] \,.} 
~~~~~~~~~~~~~~~~~~~~~~~~~~~~~~~~~~~~~~~~~~~~~~~~~~~~~~~~~~~~~~~~~~~~~~~~~~~ & & 
\\[-18pt]
\end{eqnarray*}
It formulates estimation of the first statistic above as a binary CPE problem (solved via binary logistic regression), and estimation of the remaining statistics as $s$ multiclass CPE problems (one for each $j\in[s]$), each with $s+1$ classes (solved via multiclass logistic regression).
In practice, since the label vectors $\y$ are typically sparse (only a small subset of the $s$ labels are active in any instance), the effective number of classes for each of the $s$ problems is much smaller than $s+1$, and \citet{Dembczynski2013} exploit this fact by considering the statistics 
$\P(\|\y\|_1=k, y_j=1 \,|\, x)$ only for small $k$ (based on the maximum number of active labels in the training instances).

As the proof of \Thm{thm:regret-bound} makes clear, our algorithm can be viewed as estimating the vector $\q(x)\in[0,1]^{s^2+1}$, with estimation of each component formulated as a binary CPE problem; in particular, having learned a score vector $\f_S:\X\>\R^{s^2+1}$, our algorithm yields $\gamma^{-1}(\f_S(x)) \in[0,1]^{s^2+1}$ as an estimate for $\q(x)$.
% (where $\gamma^{-1}$ is applied component-wise). 
A closer look reveals that $\q(x)$ captures essentially the same $s^2+1$ statistics as above:\footnote{Note that for each $j\in[s]$, the $s+1$ probabilities $\P(\|\y\|_1=k, y_j=1 \,|\, x)$ ($k\in [s]$) and $\P(y_j=0 \,|\, x)$ estimated by the $j$-th multiclass problem in EFP add up to 1, so the EFP algorithm effectively estimates a total of $s^2+1$ statistics.}
\vspace{-2pt}
\begin{eqnarray*}
q_{0}(x) 
	& \hspace{-8pt} = & \hspace{-8pt} 
	\E_{\y|x}[a_{\y,0}] 
	~\, =  
	\P(\|\y\|_1 = 0 \,|\, x) 
\\
q_{jk}(x) 
	& \hspace{-8pt} = & \hspace{-8pt}
	\E_{\y|x}[a_{\y,jk}] 
	= 
	\P(\|\y\|_1 = k, y_j = 1 \,|\, x) 
	\,,
	~~ j,k\in[s]
	\,.
\\[-16pt]	
\end{eqnarray*}
Thus, both algorithms effectively estimate the same statistics of the conditional label distribution $p(\y|x)$; indeed, these are precisely the statistics needed to compute a Bayes optimal multi-label classifier for the $F_\beta$-measure \cite{Dembczynski2011}. 
%Our analysis, which is motivated primarily from a convex calibrated surrogates perspective, directly yields a quantitative regret transfer bound. 
In practice, as with the EFP algorithm, our algorithm can also be implemented to estimate $q_{jk}(x)$ only for small values of $k$ (i.e.\ values of $k$ for which labelings $\y$ with $\|\y\|_1=k$ are actually seen in the training data).

\vspace{-4pt}
%========== SECTION 7 ===========
\section{Experiments}
\label{sec:expts}
\vspace{-2pt}

We conducted two sets of experiments to evaluate our algorithm. 
In the first experiment, we generated synthetic data from a known distribution for which 
%we could implement a Bayes optimal multi-label classifier, 
the Bayes optimal $F_1$-accuracy could be estimated, 
and tested the convergence of our algorithm to this optimal $F_1$ performance.
In the second set of experiments, we compared the performance of our algorithm to that of other algorithms on various benchmark data sets.
% drawn from the Mulan repository\footnote{http://mulan.sourceforge.net/datasets-mlc.html}. 
We summarize both sets of experiments below.

\vspace{-4pt}
%-----------------
\subsection{Synthetic Data: Convergence to Bayes Optimal $F_1$}%-Accuracy}
\label{subsec:expts-synthetic}
\vspace{-2pt}

In the first experiment, we tested the consistency behavior of our algorithm on a synthetic data set from a known distribution for which the Bayes optimal $F_1$ performance could be estimated.
Specifically, we generated a multi-label data set with instances $\x$ in $\X=\R^{100}$ and $s=6$ labels/tags (i.e., labelings $\y$ in $\{0,1\}^6$), such that the vector $\q(\x)\in[0,1]^{37}$ containing the $s^2+1=37$ statistics of 
%$(p(\y|\x))_{\y\in\{0,1\}^6} \in \Delta_{\{0,1\}^6}$ 
the conditional label distribution $p(\y|x)$
needed to compute a Bayes optimal multi-label classifier for $F_1$ (see \Eqn{eqn:q}) could be obtained from a linear function of $\x$. 
More precisely, we fixed a matrix $\W\in[0,1]^{37\times 100}$ with entries drawn uniformly at random from $[0,1]$; we checked that $\W$ has full row rank. We also fixed a vector $\balpha\in[0.1,1]^{64}$ with entries drawn uniformly from $[0.1,1]$. To generate a data point $(\x,\y)$, we then did the following: we first sampled $\p\in\Delta_{64}\equiv\Delta_{\{0,1\}^6}$ from $\textrm{Dirichlet}(\balpha)$. We set $\q=\E_{\y\sim\p}[\a_\y] \in [0,1]^{37}$, where $\a_\y\in\{0,1\}^{37}$ is as defined in \Eqn{eqn:a-y}.
We then took $\x = \W^\dagger \gamma_{\log}(\q)$, and drew $\y\sim\p$ (here $\W^\dagger$ denotes the pseudo-inverse of $\W$).
It can be verified that this gives $\q(\x) =\q = \gamma_{\log}^{-1}(\W\x)$, and therefore, taking the function class $\F$ in our algorithm to be the class of linear functions (i.e., functions of the form $\x\mapsto\V\x$ for $\V\in\R^{37\times 100}$) suffices to learn a Bayes optimal multi-label classifier.

With the above settings, we used our algorithm (with logistic binary loss $\phi_{\log}$ and linear function class) to learn a multi-label classifier from increasingly large training samples drawn according to the above distribution, and measured the $F_1$ performance on a large test set of $15,000$ data points drawn from the same distribution. The results are shown in \Fig{fig:consistency}. As can be seen, our algorithm indeed converges to a Bayes optimal classifier for $F_1$.

\begin{figure}[t]
\begin{center}
\scalebox{0.25}{\includegraphics{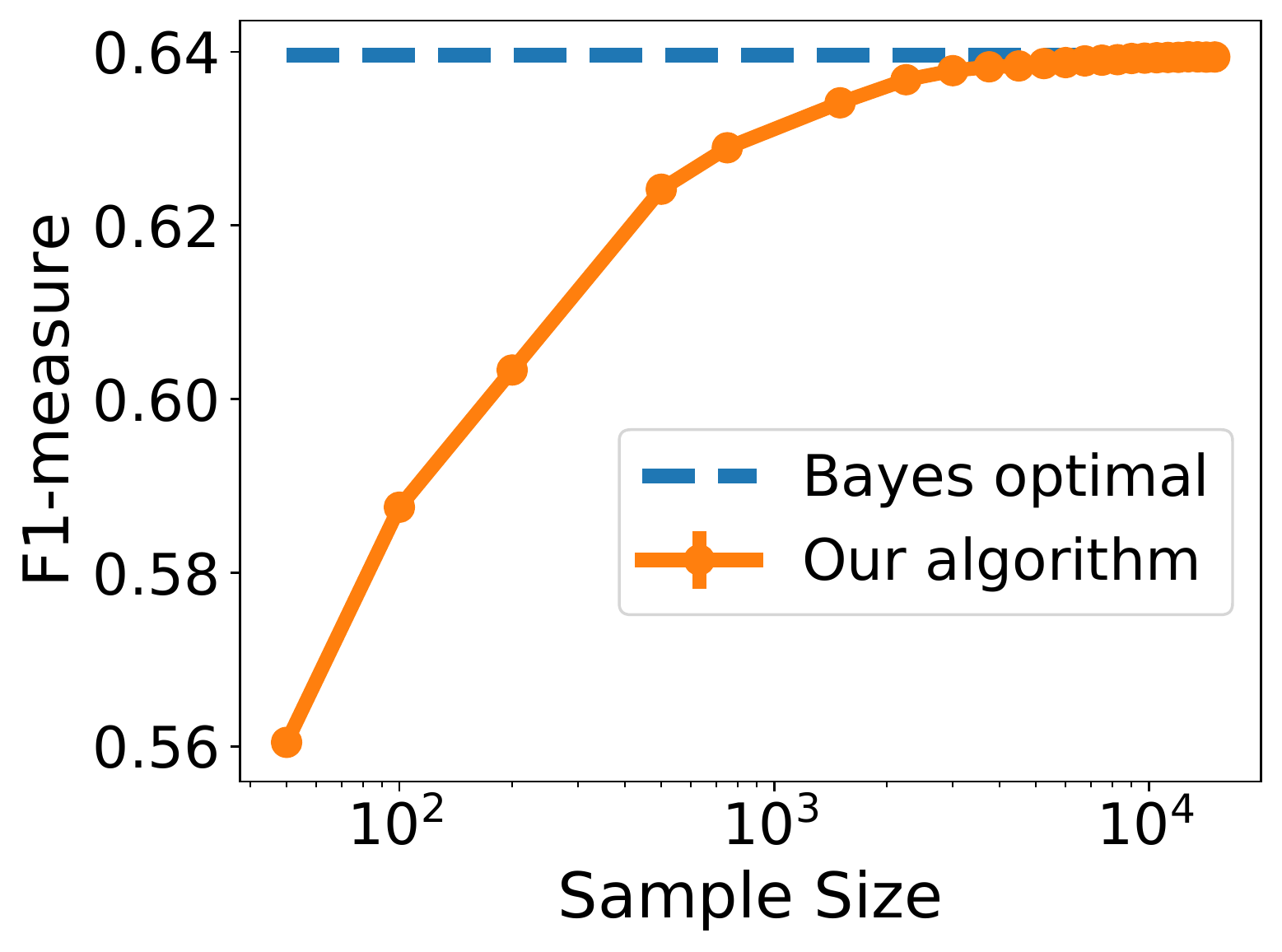}}
\vspace{-12pt}
\caption{Convergence of our algorithm to Bayes optimal $F_1$ performance on synthetic multi-label data (see \Sec{subsec:expts-synthetic}).}
\label{fig:consistency}
\vspace{-8pt}
\end{center}
\end{figure}

\vspace{-4pt}
%-----------------
\subsection{Real Data: Comparison with Other Algorithms}
\label{subsec:expts-real}
\vspace{-2pt}

In the second set of experiments, we evaluated the performance of our algorithm on various benchmark multi-label data sets drawn from the Mulan repository\footnote{http://mulan.sourceforge.net/datasets-mlc.html}. Details of the data sets are provided in \Tab{tab:data-sets}. All the data sets come with prescribed train/test splits. After training our models on the training set, we measure the \emph{instance-averaged} $F_1$ performance on the test set (i.e., we compute the multi-label $F_1$-measure on each test example and take the average).

We compared with the following algorithms: EFP \cite{Dembczynski2013}, LIMO (label-wise version recommended for instance-averaged $F_1$) \cite{Wu2017}, and BR (which treats the $s$ labels as conditionally independent and trains $s$ binary logistic regression classifiers, one for each label).
All algorithms were trained to learn linear models. Regularization parameters (for regularized logistic regression in our algorithm, EFP, and BR; and for the margin-based objective in LIMO) were chosen by 5-fold cross-validation on the training set from $\{10^{-4},\ldots,10^{3}\}$ (for all algorithms, the parameter value maximizing average $F_1$-measure across the 5 folds was selected).
For our algorithm and EFP, as discussed in \Sec{sec:relationship}, we generally implemented the algorithms to estimate only a small subset of the $s^2+1$ statistics in $\q(x)$ (only those corresponding to numbers of active labels seen in the training data); for the Birds data set, this resulted in poor performance for both algorithms, and so for this data set we trained both algorithms to perform a full estimation of all $s^2+1$ statistics.

The results are shown in \Tab{tab:results} (the asterisks in the results for the Birds data set denote the full estimation of $s^2+1$ statistics for this data set, as discussed above). 
As expected, the performance of our algorithm is similar to that of EFP. BR, as expected, is generally a relatively weak baseline. 
LIMO is sometimes competitive, but since it aims to simultaneously optimize several multi-label performance measures, we do not expect it to outperform algorithms designed for a specific performance measure, and indeed this is borne out in our experiments.

\begin{table}[t]
\vspace{-6pt}
\small
\begin{center}
\caption{Multi-label data sets used in experiments in \Sec{subsec:expts-real}.}
\label{tab:data-sets}
\vspace{1pt}
\begin{tabular}{@{~~}lrrrr@{~~}}
\hline
\textbf{{Data set}}		& \textbf{{\# train}}	& \textbf{{\# test}} 	& \textbf{{\# labels}} & \textbf{{\# features}} \\
\hline
Scene \rule{0pt}{9pt}	 	& {1211} 				& {1196}		& {6} & 294 \\
Yeast				& {1500} 				& {917}		& {14} & 103 \\
Birds					& 322 & 323 & 19 & 260 \\
Medical 	 			& {333} 				& {645}			& {45} & 1449 \\
Enron		 		& {1123} 				& {579}		& {53} & 1001 \\
Mediamill 		 		& {30993} 			& {12914}		& {101} & 120 \\
\hline
\end{tabular}
\end{center}
%\vspace{-12pt}
\vspace{-16pt}
\end{table}

\begin{table}[t]
\vspace{-6pt}
\small
\begin{center}
\caption{Comparison of $F_1$ performance of our algorithm with other MLC algorithms on various Mulan multi-label data sets. Higher values are better. See \Sec{subsec:expts-real} for details and for an explanation of the asterisks for the Birds data set.}
\label{tab:results}
\vspace{1pt}
\begin{tabular}{@{~~}lrrrr@{~~}}
\hline
\textbf{{Data set}}		& \textbf{{Our algorithm}}	& \textbf{{EFP}} 	& \textbf{{LIMO}} & \textbf{BR} \\
\hline
Scene \rule{0pt}{9pt}	 	& \textbf{0.7445} & {0.7426} & {0.6325} & {0.6009} \\
Yeast				& \textbf{0.6571} & {0.6558} & {0.4914} & {0.6065} \\
Birds					& *\textbf{0.5836} & *{0.5293} & 0.5463 & 0.5510 \\
Medical 	 			& {0.7557} & \textbf{0.7685} & 0.7237 & 0.6507 \\
Enron		 		& 0.5868 & \textbf{0.6204} & 0.5764 & 0.5455 \\
Mediamill 		 		& \textbf{0.5642} & 0.5600 & 0.5135 & 0.5229 \\
\hline
\end{tabular}
\end{center}
\vspace{-12pt}
%\vspace{-6pt}
\end{table}

\vspace{-6pt}
%========== SECTION 8 ===========
\section{Conclusion}
\label{sec:concl}
\vspace{-2pt}

We have provided a family of convex calibrated surrogate losses for the multi-label $F_\beta$-measure, together with a quantitative regret transfer bound. Our surrogates effectively decompose the $F_\beta$ learning problem over $s$ labels into (at most) $s^2+1$ binary class probability estimation (CPE) problems. 
%Minimizing the surrogates in a universal function class (with suitable regularization) yields a consistent learning algorithm whose $F_\beta$-regret goes to zero as the training sample size increases. 
The regret transfer bound allows us to transfer any regret guarantees on the binary CPE learners to regret guarantees on the overall $F_\beta$ learner.
Although motivated from a different viewpoint, like the EFP algorithm of \citet{Dembczynski2013}, our algorithm can also be viewed as a type of `plug-in' algorithm for the $F_\beta$-measure. 
While we have described the algorithm in the context of multi-label classification, the algorithm can also be used for binary sequence labeling tasks where the $F_\beta$-measure is useful.

%\newpage
\clearpage
%=============== ACKNOWLEDGMENTS ==================
\section*{Acknowledgments}

%MYZ and SA are supported in part by the US National Science Foundation (NSF) under NSF HDR TRIPODS grant 
This material is based upon work supported in part by the US National Science Foundation (NSF) under Grant Nos.\ 1934876 and 1717290 (awarded to SA). SA is also supported in part by the US National Institutes of Health (NIH) under Grant No.\ U01CA214411. HGR thanks the Robert Bosch Center for Data Science and Artificial Intelligence at IITM for its support. Any opinions, findings, and conclusions or recommendations expressed in this material are those of the authors and do not necessarily reflect the views of the National Science Foundation, the National Institutes of Health, or the Robert Bosch Center.

%\newpage
%\clearpage
%=============== BIBLIOGRAPHY ==================

\bibliography{refs-multilabel,refs-surrogates}

\begin{thebibliography}{26}
\providecommand{\natexlab}[1]{#1}
\providecommand{\url}[1]{\texttt{#1}}
\expandafter\ifx\csname urlstyle\endcsname\relax
  \providecommand{\doi}[1]{doi: #1}\else
  \providecommand{\doi}{doi: \begingroup \urlstyle{rm}\Url}\fi

\bibitem[Agarwal(2014)]{Agarwal14}
Agarwal, S.
\newblock Surrogate regret bounds for bipartite ranking via strongly proper
  losses.
\newblock \emph{Journal of Machine Learning Research}, 15:\penalty0 1653--1674,
  2014.

\bibitem[Bartlett et~al.(2006)Bartlett, Jordan, and McAuliffe]{Bartlett+06}
Bartlett, P.~L., Jordan, M., and McAuliffe, J.
\newblock Convexity, classification and risk bounds.
\newblock \emph{Journal of the American Statistical Association},
  101(473):\penalty0 138--156, 2006.

\bibitem[Dembczynski et~al.(2010)Dembczynski, Cheng, and
  H{\"{u}}llermeier]{Dembczynski2010}
Dembczynski, K., Cheng, W., and H{\"{u}}llermeier, E.
\newblock Bayes optimal multilabel classification via probabilistic classifier
  chains.
\newblock In \emph{Proceedings of the 27th International Conference on Machine
  Learning (ICML)}, pp.\  279--286, 2010.

\bibitem[Dembczynski et~al.(2011)Dembczynski, Waegeman, Cheng, and
  H{\"{u}}llermeier]{Dembczynski2011}
Dembczynski, K., Waegeman, W., Cheng, W., and H{\"{u}}llermeier, E.
\newblock An exact algorithm for {F}-measure maximization.
\newblock In \emph{Advances in Neural Information Processing Systems 24}, pp.\
  1404--1412, 2011.

\bibitem[Dembczynski et~al.(2013)Dembczynski, Jachnik, Kotlowski, Waegeman, and
  H{\"{u}}llermeier]{Dembczynski2013}
Dembczynski, K., Jachnik, A., Kotlowski, W., Waegeman, W., and
  H{\"{u}}llermeier, E.
\newblock Optimizing the {F}-measure in multi-label classification: Plug-in
  rule approach versus structured loss minimization.
\newblock In \emph{Proceedings of the 30th International Conference on Machine
  Learning (ICML)}, pp.\  1130--1138, 2013.

\bibitem[Duchi et~al.(2010)Duchi, Mackey, and Jordan]{Duchi+10}
Duchi, J., Mackey, L., and Jordan, M.
\newblock On the consistency of ranking algorithms.
\newblock In \emph{Proceedings of the International Conference on Machine
  Learning (ICML)}, 2010.

\bibitem[Gao \& Zhou(2013)Gao and Zhou]{Gao2013}
Gao, W. and Zhou, Z.
\newblock On the consistency of multi-label learning.
\newblock \emph{Artificial Intelligence}, 199-200:\penalty0 22--44, 2013.

\bibitem[Joachims(2005)]{Joachims05}
Joachims, T.
\newblock A support vector method for multivariate performance measures.
\newblock In Raedt, L.~D. and Wrobel, S. (eds.), \emph{Proceedings of the 22nd
  International Conference on Machine Learning (ICML)}, pp.\  377--384, 2005.

\bibitem[Koyejo et~al.(2015)Koyejo, Natarajan, Ravikumar, and
  Dhillon]{Koyejo2015}
Koyejo, O., Natarajan, N., Ravikumar, P., and Dhillon, I.~S.
\newblock Consistent multilabel classification.
\newblock In \emph{Advances in Neural Information Processing Systems 28}, pp.\
  3321--3329, 2015.

\bibitem[Parambath et~al.(2014)Parambath, Usunier, and
  Grandvalet]{Parambath2014}
Parambath, S.~P., Usunier, N., and Grandvalet, Y.
\newblock Optimizing {F}-measures by cost-sensitive classification.
\newblock In \emph{Advances in Neural Information Processing Systems 27}, pp.\
  2123--2131, 2014.

\bibitem[Petterson \& Caetano(2010)Petterson and Caetano]{PettersonCa10}
Petterson, J. and Caetano, T.~S.
\newblock Reverse multi-label learning.
\newblock In \emph{Advances in Neural Information Processing Systems 23}, pp.\
  1912--1920. 2010.

\bibitem[Petterson \& Caetano(2011)Petterson and Caetano]{PettersonCa11}
Petterson, J. and Caetano, T.~S.
\newblock Submodular multi-label learning.
\newblock In \emph{Advances in Neural Information Processing Systems 24}, pp.\
  1512--1520. 2011.

\bibitem[Pillai et~al.(2017)Pillai, Fumera, and Roli]{Pillai2017}
Pillai, I., Fumera, G., and Roli, F.
\newblock Designing multi-label classifiers that maximize {F} measures: State
  of the art.
\newblock \emph{Pattern Recognition}, 61:\penalty0 394--404, 2017.

\bibitem[Ramaswamy(2015)]{Ramaswamy15-thesis}
Ramaswamy, H.~G.
\newblock \emph{Design and Analysis of Consistent Algorithms for Multiclass
  Learning Problems}.
\newblock PhD thesis, Indian Institute of Science, 2015.

\bibitem[Ramaswamy et~al.(2013)Ramaswamy, Agarwal, and Tewari]{Ramaswamy+13}
Ramaswamy, H.~G., Agarwal, S., and Tewari, A.
\newblock Convex calibrated surrogates for low-rank loss matrices with
  applications to subset ranking losses.
\newblock In \emph{Advances in Neural Information Processing Systems}, 2013.

\bibitem[Ramaswamy et~al.(2014)Ramaswamy, Babu, Agarwal, and
  Williamson]{Ramaswamy+14}
Ramaswamy, H.~G., Babu, B.~S., Agarwal, S., and Williamson, R.~C.
\newblock On the consistency of output code based learning algorithms for
  multiclass learning problems.
\newblock In \emph{Proceedings of the 27th Conference on Learning Theory
  (COLT)}, pp.\  885--902, 2014.

\bibitem[Ramaswamy et~al.(2015)Ramaswamy, Tewari, and Agarwal]{Ramaswamy+15}
Ramaswamy, H.~G., Tewari, A., and Agarwal, S.
\newblock Convex calibrated surrogates for hierarchical classification.
\newblock In \emph{Proceedings of the 32nd International Conference on Machine
  Learning (ICML)}, pp.\  1852--1860, 2015.

\bibitem[Reid \& Williamson(2010)Reid and Williamson]{ReidWi10}
Reid, M.~D. and Williamson, R.~C.
\newblock Composite binary losses.
\newblock \emph{Journal of Machine Learning Research}, 11:\penalty0 2387--2422,
  2010.

\bibitem[Steinwart(2007)]{Steinwart07}
Steinwart, I.
\newblock How to compare different loss functions and their risks.
\newblock \emph{Constructive Approximation}, 26:\penalty0 225--287, 2007.

\bibitem[Tewari \& Bartlett(2007)Tewari and Bartlett]{TewariBa07}
Tewari, A. and Bartlett, P.~L.
\newblock On the consistency of multiclass classification methods.
\newblock \emph{Journal of Machine Learning Research}, 8:\penalty0 1007--1025,
  2007.

\bibitem[Tsochantiridis et~al.(2005)Tsochantiridis, Joachims, Hoffman, and
  Altun]{Tsochantiridis+05}
Tsochantiridis, I., Joachims, T., Hoffman, T., and Altun, Y.
\newblock Large margin methods for structured and interdependent output
  variables.
\newblock \emph{Journal of Machine Learning Research}, 6:\penalty0 1453--1484,
  2005.

\bibitem[Wu \& Zhou(2017)Wu and Zhou]{Wu2017}
Wu, X. and Zhou, Z.
\newblock A unified view of multi-label performance measures.
\newblock In \emph{Proceedings of the 34th International Conference on Machine
  Learning (ICML)}, pp.\  3780--3788, 2017.

\bibitem[Ye et~al.(2012)Ye, Chai, Lee, and Chieu]{Ye2012}
Ye, N., Chai, K. M.~A., Lee, W.~S., and Chieu, H.~L.
\newblock Optimizing {F}-measure: {A} tale of two approaches.
\newblock In \emph{Proceedings of the 29th International Conference on Machine
  Learning (ICML)}, 2012.

\bibitem[Zhang \& Zhou(2014)Zhang and Zhou]{Zhang2014}
Zhang, M. and Zhou, Z.
\newblock A review on multi-label learning algorithms.
\newblock \emph{{IEEE} Transactions on Knowledge and Data Engineering},
  26\penalty0 (8):\penalty0 1819--1837, 2014.

\bibitem[Zhang(2004{\natexlab{a}})]{Zhang04a}
Zhang, T.
\newblock Statistical behavior and consistency of classification methods based
  on convex risk minimization.
\newblock \emph{Annals of Statistics}, 32(1):\penalty0 56--134,
  2004{\natexlab{a}}.

\bibitem[Zhang(2004{\natexlab{b}})]{Zhang04b}
Zhang, T.
\newblock Statistical analysis of some multi-category large margin
  classification methods.
\newblock \emph{Journal of Machine Learning Research}, 5:\penalty0 1225--1251,
  2004{\natexlab{b}}.

\end{thebibliography}
\bibliographystyle{icml2020}

\clearpage
%================== APPENDIX ===================
\onecolumn
\appendix

\begin{center}
\textbf{\Large Convex Calibrated Surrogates for the Multi-Label F-Measure}
\\[8pt]
\textbf{\Large Supplementary Material} 
\\[8pt]
\end{center}

\section*{Implementation of `decode'}

In order to 
%implement the mapping $\decode:\R^{s^2+1}\>\{0,1\}^s$ defined in \Eqn{eqn:decode} 
solve the combinatorial optimization problem involved in the mapping $\decode:\R^{s^2+1}\>\{0,1\}^s$ as defined in \Eqn{eqn:decode} 
efficiently, we make use of an $O(s^3)$-time procedure due to \citet{Dembczynski2011}. Specifically, \citet{Dembczynski2011} gave a procedure that, given a certain set of $s^2+1$ statistics of the true conditional distribution $p(\y|x)$ at a point $x\in\X$, computes in $O(s^3)$ time a Bayes optimal multi-label prediction $h^*(x)\in\{0,1\}^s$ at that point with respect to the $F_1$-measure by solving a similar combinatorial optimization problem (the approach generalizes easily to the $F_\beta$-measure for general $\beta$).
%As discussed in \Sec{sec:relationship}, o
Our algorithm (\Algo{algo}) can be viewed as effectively estimating the same $s^2+1$ statistics from the training sample $S$; in particular, once a scoring function $\f_S:\X\>\R^{s^2+1}$ is learned by minimizing our surrogate loss $\psi$, the estimated statistics at a point $x\in\X$ are given by $\gamma^{-1}(\f_S(x))$ (where $\gamma^{-1}$ is the inverse of the link function $\gamma:[0,1]\>\R$ associated with the strictly proper composite binary loss $\phi$ used in our surrogate, and is applied element-wise to $\f_S(x)$).
Our `decode' mapping effectively corresponds to estimating a Bayes optimal prediction at $x$ using these estimated statistics;
we can therefore apply the procedure of \citet{Dembczynski2011} to these estimated statistics.
%\footnote{The plug-in algorithm of \citet{Dembczynski2013} also applies the same procedure to its estimated statistics.}

The implementation below is described for a general input vector $\u\in\R^{s^2+1}$ (see \Eqn{eqn:decode}); in our $F_\beta$ learning algorithm, to make a prediction at $x\in\X$, it would be applied to $\u=\f_S(x)$.
The overall idea is that the combinatorial search over $\hat{\y}\in \{0,1\}^s$ is stratified over the $s+1$ sets $\hat{\Y}_l = \{\hat{\y} \in\{0,1\}^s: \|\hat{\y}\|_1 = l\}$, $l\in\{0,1,\ldots,s\}$; 
%the search within each of these sets (for $l\neq 0$) takes the form 
to find an optimal element $\hat{\y}^{l,*}$ within each set $\hat{\Y}_l$, one need only solve a problem of the form $\hat{\y}^{l,*} \in \argmin_{\hat{\y}\in\hat{\Y}_l} \sum_{j=1}^s \hat{y}_j T_{jl}$ for certain numbers $T_{jl}$, which can be done simply by finding the smallest $l$ numbers among $\{T_{jl}:j\in[s]\}$ and setting the corresponding $l$ entries of $\hat{\y}^{l,*}$ to 1 (and remaining entries to 0). Solving these $s+1$ subproblems and picking the best solution among them takes a total of $O(s^2\ln(s))$ time; computing the $s^2$ numbers $T_{jl}$ involves a matrix multiplication that takes a total of $O(s^3)$ time.\footnote{One could in principle use faster matrix multiplication methods that take $o(s^3)$ time, but in practice, this would be helpful for only extremely large values of $s$.}

\begin{algorithm}[h]
\caption{Decode}
\begin{algorithmic}[1]
\STATE \textbf{Input:} 
	Vector $\u=(u_0,(u_{jk})_{j,k=1}^s))^\top\in\R^{s^2+1}$ 
\STATE \textbf{Parameters:} 
	Link function $\gamma:[0,1]\>\R$ 
\STATE 
Define matrices $\Q\in[0,1]^{s\times s}$ and $\V\in\R^{s\times s}$ as follows:
\begin{eqnarray*}
Q_{jk} & = & \gamma^{-1}(u_{jk}) \\
V_{kl} & = & \frac{-(1+\beta)^2}{\beta^2 k + l} 
\end{eqnarray*}
\STATE 
Compute $\T=\Q\V$   ~~~\red{// matrix multiplication, $O(s^3)$ time}
\STATE
\textbf{For} $l=1\ldots s$:  ~~~~~~~~~~~~~~~~~~~~~~~~~~~~~~~~\red{// for loop takes total $O(s^2\ln(s))$ time}
\STATE
\quad\quad Find the $l$ smallest numbers among $\{T_{jl}:j\in[s]\}$; call the corresponding indices $j_1^l,\ldots,j_l^l$ 
\STATE
\quad\quad Define $\hat{\y}^{l,*}\in\{0,1\}^s$ as follows: 
\begin{eqnarray*}
\hat{y}^{l,*}_j & = & 
	\left\{ 
		\begin{array}{ll}
			1 & \text{if $j \in \{j_1^l,\ldots,j_l^l\}$} \\
			0 & \text{otherwise.} 
		\end{array}
	\right.
	~~~\red{\text{// this solves $\hat{\y}^{l,*} \in \textstyle{\argmin_{\hat{\y}\in\hat{\Y}_l} \sum_{j=1}^s \hat{y}_j T_{jl}}$}}
\end{eqnarray*}
\STATE
\quad\quad Set $z^*_l = \sum_{j=1}^s \hat{y}^{l,*}_j T_{jl}$
\STATE
\textbf{End for}
\STATE
Pick $\hat{\y}^* \in \{0,1\}^s$ as follows:
\[
\hat{\y}^* \in \argmin_{\hat{\y}\in\{\0,\,\hat{\y}^{1,*},\ldots,\,\hat{\y}^{s,*}\}} ~ 
%	b_{\hat{\y},0} \cdot \gamma^{-1}(u_0) + \sum_{j=1}^s \sum_{k=1}^s b_{\hat{\y},jk}  \cdot \gamma^{-1}(u_{jk})
%	-\1(\|\hat{\y}\|_1=0) \cdot \gamma^{-1}(u_0) - \sum_{j=1}^s \sum_{k=1}^s \frac{(1+\beta^2) \cdot \hat{y}_j}{\beta^2 k + \|\hat{\y}\|_1}  \cdot \gamma^{-1}(u_{jk})
%	-\1(\|\hat{\y}\|_1=0) \cdot \gamma^{-1}(u_0) - \sum_{j=1}^s \sum_{k=1}^s \frac{(1+\beta^2) \cdot \hat{y}_j}{\beta^2 k + \|\hat{\y}\|_1}  \cdot Q_{jk}
%	-\1(\hat{\y}=\0) \cdot \gamma^{-1}(u_0) + \1(\hat{\y}\neq \0) \sum_{j=1}^s \hat{y}_j T_{j,\|\hat{\y}\|_1}
	-\1(\hat{\y}=\0) \cdot \gamma^{-1}(u_0) + \1(\hat{\y}\neq \0) \cdot z^*_{\|\hat{\y}\|_1}
%	~~~~~~ \red{\text{// $O(s^3)$ time}}
\]
\STATE \textbf{Output:} $\hat{\y}^* \in \{0,1\}^s$
\end{algorithmic}
\label{algo}
\end{algorithm}

%%%%%%%%%%%%%%%%%%%%%%%%%%%%%%%%%%%%%%%%%%%%%%%%%%%%%%%%%%%%%%%%%%%%%%%%%%%%%%%
%%%%%%%%%%%%%%%%%%%%%%%%%%%%%%%%%%%%%%%%%%%%%%%%%%%%%%%%%%%%%%%%%%%%%%%%%%%%%%%

\end{document}